\documentclass[twoside,11pt]{article}

\usepackage{jmlr2e}
\usepackage{amsmath}
\usepackage{algorithm, algorithmic}
\usepackage{color}

\newtheorem{assumption}{Assumption}
\newtheorem{fact}{Fact}
\newtheorem{lem}{Lemma}
\newtheorem{defn}{Definition}
\newtheorem{cor}{Corollary}

\newcommand{\E}{\mathbb{E}}
\newcommand{\Prob}{\mathbb{P}}
\newcommand{\hist}{\mathcal{F}_{t}}
\newcommand{\A}{\mathcal{A}}
\newcommand{\Y}{\mathcal{Y}}

\jmlrheading{}{2014}{}{}{}{Daniel Russo and Benjamin Van Roy}

\ShortHeadings{An Information-Theoretic Analysis of Thompson Sampling}{Russo and Van Roy}
\firstpageno{1}

\begin{document}

\title{An Information-Theoretic Analysis of Thompson Sampling}

\author{\name Daniel Russo \email djrusso@stanford.edu \\
       \addr Department of Management Science and Engineering\\
       Stanford University\\
       Stanford, California 94305
       \AND
       \name Benjamin\ Van Roy \email bvr@stanford.edu \\
       \addr Departments of Management Science and Engineering and Electrical Engineering\\
       Stanford University\\
       Stanford, California 94305}

\editor{Peter Auer}

\maketitle

\begin{abstract}
We provide an information-theoretic analysis of Thompson sampling that applies across a broad range of online optimization problems in which 
a decision-maker must learn from partial feedback.  This analysis inherits the simplicity and elegance of information theory
and leads to regret bounds that scale with the entropy of the optimal-action distribution. This strengthens preexisting results and yields new insight into how information improves performance.

\end{abstract}


\section{Introduction}

This paper considers the problem of repeated decision making in the presence of model uncertainty. A decision-maker repeatedly chooses among a set of possible actions, observes an outcome, and receives a reward representing the utility derived from this outcome. The decision-maker is uncertain about the underlying system and is therefore initially unsure of which action is best. However, as outcomes are observed, she is able to learn over time to make increasingly effective decisions. Her objective is to choose actions sequentially so as to maximize the expected cumulative reward. 

We focus on settings with {\it partial feedback}, under which the decision-maker does not generally observe what the reward would have been had she selected a different action.
This feedback structure leads  to an inherent tradeoff between {\it exploration and exploitation}:  by experimenting with poorly understood actions one can learn to make more effective decisions in the future, but focusing on better understood actions may lead to higher rewards in the short term. The classical multi--armed bandit problem is an important special case of this formulation. In such problems, the decision--maker only observes rewards she receives, and rewards from one action provide no information about the reward that can be attained by selecting other actions. We are interested here primarily in models and algorithms that accommodate cases where the number of actions is very large and there is a richer {\it information structure} relating actions and observations. This category of problems is often referred to as {\it online optimization with partial feedback.} 

A large and growing literature treats the design and analysis of algorithms for such problems.  An online optimization algorithm typically starts with two forms
of prior knowledge.  The first -- {\it hard knowledge} -- posits that the mapping from action to outcome distribution lies within a particular family of mappings.  
The second -- {\it soft knowledge} -- concerns which of these mappings are more or less likely to match reality.  Soft knowledge evolves with observations
and is typically represented in terms of a probability distribution or a confidence set.

Much recent work concerning online optimization algorithms focuses on establishing performance guarantees in the form of {\it regret bounds}.
Surprisingly, virtually all of these regret bounds depend on hard knowledge but not soft knowledge, a
notable exception being the bounds of \citet{srinivas2012information} which we discuss further in Section \ref{sec:review}.
Regret bounds that depend on hard knowledge
yield insight into how an algorithm's performance scales with the complexity of the family of mappings and are useful for
delineating algorithms on that basis, but if a regret bound does not depend on soft knowledge, it does not have much to say about how future performance 
should improve as data is collected.  The latter sort of insight should be valuable for designing better ways of trading off between exploration
and exploitation, since it is soft knowledge that is refined as a decision-maker learns.  Another important 
benefit to understanding how performance depends on soft knowledge arises in practical applications: when designing an online learning
algorithm, one may have access to historical data and want to understand how this prior information benefits future performance.

In this paper, we establish regret bounds that depend on both hard and soft knowledge 
for a simple online learning algorithm alternately known as {\it Thompson sampling}, {\it posterior sampling}, or {\it probability matching}.  
The bounds strengthen results from prior work not only with respect to Thompson sampling but relative to regret bounds for {\it any} 
online optimization algorithm.  Further, the bounds
offer new insight into how regret depends on soft knowledge.   Indeed, forthcoming work of ours leverages this  
to produce an algorithm that outperforms Thompson sampling.

We found information theory to provide tools ideally suited for deriving our new regret bounds,
and our analysis inherits the simplicity and elegance enjoyed by work in that field.
Our formulation encompasses a broad family of information structures, including as special cases multi--armed bandit problems with independent arms, online optimization problems with full information, linear bandit problems, and problems with combinatorial action sets and ``semi--bandit'' feedback.
We leverage information theory to provide a unified analysis that applies to each of those special cases, establishing that
Thompson sampling satisfies order optimal Bayesian regret bounds for each one.\footnote{For two of these special cases, the bounds are only tight up to a logarithmic factor.} 

A novel feature of our bounds is their dependence on the entropy of the optimal-action distribution. To our knowledge, these are the first bounds on the expected 
regret of {\it any} algorithm that depend on the magnitude of the agent's uncertainty about which action is optimal.    The fact that our bounds only depend on 
uncertainties relevant to optimizing performance highlights the manner in which Thompson sampling naturally exploits complex 
information structures.  Further, in practical contexts, a decision-maker may begin with an understanding that some actions are more likely to
be optimal than others.  For example, when dealing with a shortest path problem, one might expect paths that traverse fewer edges to generally
incur less cost.  Our bounds are the first to formalize the performance benefits afforded by such an understanding.

\subsection{ Preview of Results }

Our analysis is based on a general probabilistic, or Bayesian, formulation in which uncertain quantities are modeled as random variables. In principle, the 
optimal strategy for such a problem could be calculated via dynamic programing, but for problem classes of practical interest this would be computationally intractable.
{\it Thompson sampling} serves as a simple and elegant heuristic strategy.
In this section, we provide a somewhat informal problem statement, and a preview of our main results about Thompson sampling. 
In the next subsection we discuss how these results relate to the existing literature. 

A decision maker repeatedly chooses among a finite set of possible actions $\A$ and upon taking action $a \in \A$ she observes a random outcome $Y_{t,a}\in \Y$. She associates a reward with each outcome as specified by a reward function $R:\Y \rightarrow \mathbb{R}$. The outcomes $\left( Y_{t,a}\right)_{t\in \mathbb{N}}$ are drawn independently over time from a fixed probability distribution $p^*_{a}$ over $\Y$. The decision maker is uncertain about the distribution of outcomes $p^*=(p^{*}_{a})_{a  \in \A}$, which is itself distributed according to a prior distribution over a family $\mathcal{P}$ of such distributions. However, she is able to learn about $p^*$ as the outcomes of past decisions are observed, and this allows her to learn over time to attain improved performance. 

We first present a special  case of our main result that applies to online linear optimization problems under bandit feedback. This result applies when each action is associated with a $d$--dimensional feature vector and the mean reward of each action is the inner product between an unknown parameter vector and the action's known feature vector. More precisely, suppose $\A \subset \mathbb{R}^d$ and that for every $p\in \mathcal{P}$ there is a vector $\theta_{p}\in \mathbb{R}^d$ such that 
\begin{equation}\label{eq: linear model}
\underset{y\sim p_a}{\E}\left[ R(y)  \right] = a^T \theta_p
\end{equation}
for all $a \in \A$. When an action is sampled, a random reward in $[0,1]$ is received, where the mean reward is given by \eqref{eq: linear model}. Then, our analysis establishes that the expected cumulative regret of Thompson sampling up to time $T$ is bounded by 
\begin{equation}\label{eq: linear bound}
\sqrt{\frac{{\rm Entropy}(A^*) dT  }{2}},
\end{equation}
where $A^* \in \underset{a\in \A}{\arg\max} \, \E \left[ R(Y_{t,a}) | p^* \right]$ denotes the optimal action. 
This bound depends on the time horizon, the entropy 
of the of the prior distribution of the optimal action $A^*$, and the dimension $d$ of the linear model. 

Because the entropy of $A^*$ is always less than $\log |\A|$, \eqref{eq: linear bound}   yields a bound of order $\sqrt{ \log(|\A|)dT }$, and this scaling cannot be improved in general (see Section \ref{subsec: linear bandit}). The bound $\eqref{eq: linear bound}$ is stronger than this worst--case bound, since the entropy of $A^*$ can be much smaller than $\log(|\A|)$.

Thompson sampling incorporates prior knowledge in a flexible and coherent way, and the benefits of this are reflected in two distinct ways by the above bound. First, as in past work \citep[see e.g.][]{dani2008stochastic, rusmevichientong2010linearly}, the bound depends on the dimension of the linear model instead of the number of actions. This reflects that the algorithm is able to learn more rapidly by exploiting the known model, since observations from selecting one action provide information about rewards that would have been generated by other actions. Second, the bound depends on the entropy of the prior distribution of $A^*$ instead of a worst case measure like the logarithm of the number of actions. This highlights the benefit of prior knowledge that some actions are more likely to be optimal than others. In particular, this bound exhibits the natural property that as the entropy of the prior distribution of $A^*$ goes to zero, expected regret does as well.  

Our main result extends beyond the class of linear bandit problems. Instead of depending on the linear dimension of the model, it depends on a more general measure of the problem's information complexity: what we call the problem's {\it information ratio}. By bounding the information ratio in specific settings, we recover the bound \eqref{eq: linear bound} as a special case, along with bounds for problems with full feedback and problems with combinatorial action sets and ``semi--bandit'' feedback. 

\subsection{Related Work}
\label{sec:review}


Though Thompson sampling was first proposed in 1933 \citep{thompson1933}, until recently it was largely ignored in the academic literature. Interest in the algorithm grew after empirical studies \citep{scott2010modern, chapelle2011empirical} demonstrated performance exceeding state of the art. Over the past several years, it has also been  adopted in industry.\footnote{Microsoft \citep{graepel2010web}, Google analytics \citep{googleanalytics} and Linkedin \citep{tang2013automatic} have used Thompson sampling.}  This has prompted a surge of interest in providing theoretical guarantees for Thompson sampling. 

One of the first theoretical guarantees for Thompson sampling was provided by \citet{may2012optimistic}, but they showed only that the algorithm converges asymptotically to optimality. \citet{agrawal2012analysis, kaufmann2012thompson,  agrawal2013further} 
and \citet{Korda2013Thompson} studied on the classical multi-armed bandit problem, where sampling one action provides no information about other actions. They provided frequentist regret bounds for Thompson sampling that are asymptotically optimal in the sense defined by \citet{lai1985asymptotically}. To attain these bounds, the authors fixed a specific uninformative prior distribution, and studied the algorithm's performance assuming this prior is used. 

Our interest in Thompson sampling is motivated by its ability to incorporate rich forms of prior knowledge about the actions and the relationship among them. Accordingly, we study the algorithm in a very general framework, allowing for an {\it arbitrary} prior distribution over the true outcome distributions $p^* = (p^{*}_{a})_{a \in \A}$. To accommodate this level of generality while still focusing on finite--time performance, we study the algorithm's {\it expected regret} under the prior distribution. This measure is sometimes called {\it Bayes risk} or {\it Bayesian regret}. 

Our recent work \citep{russo2013} provided the first results in this setting. That work leverages a close connection between Thompson sampling and upper confidence bound (UCB) algorithms, as well as existing analyses of several UCB algorithms. This confidence bound analysis was then extended to a more general setting, leading to a general regret bound stated in terms of a new notion of model complexity -- what we call the eluder dimension. While the connection with UCB algorithms may be of independent interest, it's desirable to have a simple, self--contained, analysis that does not rely on the--often complicated--construction of confidence bounds.

\citet{agrawal2013further} provided the first ``distribution independent'' bound for Thompson sampling. They showed that when Thompson sampling is executed with an independent uniform prior and rewards are binary the algorithm satisfies a frequentist regret bound\footnote{They bound regret conditional on the true reward distribution: $\E \left[ {\rm Regret}(T, \pi^{\rm TS}) | p^* \right]$.} of order $\sqrt{|\A|T \log(T)}$.
\citet{russo2013} showed that, for an arbitrary prior over bounded reward distributions, the expected regret of Thompson sampling under this prior is bounded by a term of order $\sqrt{|\A|T \log(T)}$. \citet{bubeck2013prior} showed that this second bound can be improved to one of order $\sqrt{|\A|T}$ using more sophisticated confidence bound analysis, and also studied a problem setting where the regret of Thompson sampling is bounded uniformly over time. In this paper, we are interested mostly in results that replace the explicit $\sqrt{|\A|}$ dependence on the number of actions with a more general measure of the problem's information complexity. For example, as discussed in the last section, for the problem of linear optimization under bandit feedback one can provide bounds that depend on the dimension of the linear model instead of the number of actions. 

To our knowledge, \citet{agrawal2013thompson} are the only other authors to have studied the application of Thompson sampling to linear bandit problems. They showed that, when the algorithm is applied with an uncorrelated Gaussian prior over $\theta_{p^*}$, it satisfies frequentist regret bounds of order $\frac{d^2}{\epsilon}\sqrt{T^{1+\epsilon}}(\log(Td)$. Here $\epsilon$ is a parameter used by the algorithm to control how quickly the posterior concentrates. \citet{russo2013} allowed for an arbitrary prior distribution, and provided a bound on expected regret (under this prior) of order $d \log(T)\sqrt{T}.$ Unlike the bound \eqref{eq: linear bound}, these results hold whenever $\A$ is a compact subset of $\mathbb{R}^{d}$, but we show in Appendix \ref{sec: infinite action spaces} that through discretization argument the bound  \eqref{eq: linear bound} also yields a similar bound whenever  $\A$ is compact. In the worst case, that bound is of order $d\sqrt{T \log(T)}$. 
%

Other very recent work \citep{gopalan2014thompson} provided a general asymptotic guarantee for Thompson sampling. They studied the growth rate of the cumulative number of times a suboptimal action is chosen as the time horizon $T$ tends to infinity. One of their asymptotic results applies to the problem of online linear optimization under ``semi--bandit'' feedback, which we study in Section \ref{subsec: semi-bandit}.

An important aspect of our regret bound is its dependence on soft knowledge through the entropy of the optimal-action distribution. One of the only other regret bounds that depends on soft--knowledge was provided very recently by \citet{li2013generalized}. Inspired by a connection between Thompson sampling and exponential weighting schemes, that paper introduced a family of Thompson sampling like algorithms and studied their application to contextual bandit problems. While our analysis does not currently treat contextual bandit problems, we improve upon their regret bound in several other respects. First, their bound depends on the entropy of the prior distribution of mean rewards, which is never smaller, and can be much larger, than the entropy of the distribution of the optimal action. In addition, their bound has an order $T^{2/3}$ dependence on the problem's time horizon, and, in order to guarantee each action is explored sufficiently often, requires that actions are frequently selected uniformly at random. In contrast, our focus is on settings where the number of actions is large and the goal is to learning without sampling each one. 

Another regret bound that to some extent captures dependence on soft knowledge is that of \citet{srinivas2012information}.
This excellent work focuses on extending algorithms and expanding theory to address multi-armed bandit problems with arbitrary reward functions 
and possibly an infinite number of actions.  In a sense, there is no hard knowledge while soft knowledge is represented in terms of a Gaussian process over a possibly infinite 
dimensional family of functions.  An upper-confidence-bound algorithm is proposed and analyzed. Our earlier work \citep{russo2013} showed similar bounds also apply to Thompson sampling.  Though our results here do not treat infinite action spaces,
it should be possible to extend the analysis in that direction.  One substantive difference is that our results apply to a much broader class of models: 
distributions are not restricted to Gaussian and more complex information structures are allowed.  Further, the results of \citet{srinivas2012information}
do not capture the benefits of soft knowledge to the extent that ours do.  For example, their regret bounds do not depend on the mean of the reward function 
distribution, even though mean rewards heavily influence the chances that each action is optimal.  Our regret bounds, in contrast, establish that regret
vanishes as the probability that a particular action is optimal grows. 

Our review has discussed the recent literature on Thompson sampling as well as two papers that have established regret bounds that depend on soft knowledge.  There is of course an immense body of work on alternative approaches to efficient exploration, including work on the Gittins index approach \citep{gittins2011multi}, Knowledge Gradient strategies \citep{ryzhov2012knowledge}, and upper-confidence bound strategies \citep{lai1985asymptotically,  auer2002finite}. In an adversarial framework, authors often study exponential-weighting shemes or, more generally, strategies based on online stochastic mirror descent. \citet{bubeck2012regret} provided a general review of upper--confidence strategies and algorithms for the adversarial multi-armed bandit problem. 

\section{Problem Formulation}
The decision--maker sequentially chooses actions $(A_t)_{t\in \mathbb{N}}$ from the action set $\A$ and observes the corresponding outcomes $\left(Y_{t, A_t}\right)_{t\in \mathbb{N}}$. To avoid measure-theoretic subtleties, we assume the space of possible outcomes $\Y$ is a subset of a finite dimensional Euclidean space.  There is a random outcome $Y_{t, a} \in \Y$ associated with each $a\in \A$ and time $t \in \mathbb{N}$.  Let $Y_t \equiv (Y_{t,a})_{a \in \A}$ be the vector of outcomes at time $t \in \mathbb{N}$.  The ``true outcome distribution'' $p^*$ is a distribution over $\Y^{|\A|}$ that is itself randomly drawn from the family of distributions $\mathcal{P}$.
We assume that, conditioned on $p^*$, $(Y_t)_{t \in \mathbb{N}}$ is an iid sequence with each element $Y_t$ distributed according to $p^*$.  Let $p_a^*$ be the marginal distribution corresponding to $Y_{t,a}$.

%

The agent associates a reward $R(y)$ with each outcome $y\in \Y$, where the reward function $R: \Y \rightarrow \mathbb{R}$ is fixed and known.  Uncertainty about $p^*$ induces uncertainty about the true optimal action, which we denote by $A^* \in \underset{a \in \A}{\arg \max}\,\E\left[ R(Y_{t,a}) | p^*\right] $. The $T$ period {\it regret} of the sequence of actions $A_1, .., A_T$ is the random variable, 
\begin{equation}\label{eq: regret}
{\rm Regret}(T) := \sum_{t=1}^{T} \left[ R(Y_{t,A^*}) - R(Y_{t, A_t}) \right],
\end{equation}
which measures the cumulative difference between the reward earned by an algorithm that always chooses the optimal action, and actual accumulated reward up to time $T$. In this paper we study expected regret
\begin{equation}\label{eq: expected regret}
\E \left[ {\rm Regret}(T)\right] = \E \left[\sum_{t=1}^{T}  \left[ R(Y_{t, A^*}) - R(Y_{t, A_t}) \right] \right], 
\end{equation}
where 
the expectation is taken over the randomness in the actions $A_t$ and the outcomes $Y_t$, and over the prior distribution over $p^*$. This measure of performance is commonly called {\it Bayesian regret} or {\it Bayes risk}.  

\paragraph{Filtrations and randomized policies:}

We define all random variables with respect to a probability space $(\Omega, \mathcal{F}, \Prob)$. Actions are chosen based on the history of past observations and possibly some external source of randomness. To represent this external source of randomness more formally, we introduce a sequence of random variables $\left(U_t\right)_{t \in \mathbb{N}}$, where for each $i\in \mathbb{N}$, $U_i$ is jointly independent of $\{U_{t}\}_{t\neq i}$, the outcomes $\left\{Y_{t,a}\right\}_{t \in \mathbb{N}, a \in \A}$, and $p^*$. We fix the filtration $\left( \mathcal{F}_{t}\right)_{t\in \mathbb{N}}$  where $\hist \subset \mathcal{F}$  is the sigma--algebra generated by $\left( A_{1}, Y_{1,A_1},..., A_{t-1}, Y_{t-1, A_{t-1}})  \right)$. 
 The action $A_t$ is  measurable with respect to the sigma--algebra generated by $\left(\hist, U_t \right)$. That is, given the history of past observations, $A_t$ is random only through its dependence on $U_{t}$. 

The objective is to choose actions in a manner that minimizes expected regret. For this purpose,  it's useful to think of the actions as being chosen by a {\it randomized policy} $\pi$, which is an $\mathcal{F}_{t}$--adapted sequence $\left(\pi_t\right)_{t\in \mathbb{N}}$. An action is chosen at time $t$ by randomizing according to $\pi_t(\cdot)=\Prob(A_t \in \cdot \vert \hist)$, which specifies a probability distribution over $\A$. We explicitly display the dependence of regret on the policy $\pi$, letting $\E\left[{\rm Regret}(T, \pi)\right]$ denote the expected value given by \eqref{eq: expected regret} when the actions $(A_{1},..,A_{T})$ are chosen according to $\pi$. 

\paragraph{Further Assumptions:}
To simplify the exposition, our main results will be stated under two further assumptions. The first requires that rewards are uniformly bounded, effectively controlling the worst-case variance of the reward distribution. In particular, this assumption is used only in proving Fact \ref{fact: DME to DKL}. In the technical appendix, we show that Fact \ref{fact: DME to DKL} can be extended to the case where reward distributions are sub-Gaussian, which yields results in that more general setting. 
\begin{assumption}\label{assum: bounded rewards}
$\underset{\overline{y}\in \Y}{\sup} R(\overline{y})-\underset{\underline{y}\in \Y}{\inf}R(\underline{y}) \leq 1.$
\end{assumption}
Our next assumption requires that the action set is finite.  In the technical appendix we  show that some cases where $\A$ is infinite can be addressed through  discretization arguments.  
\begin{assumption}
$\A$ is finite. 
\end{assumption}
Because the Thompson sampling algorithm only chooses actions from the support of $A^*$, all of our results hold when the finite set $\A$ denotes only the actions in the support of $A^*$. This difference can be meaningful. For example, when $\A$ is a polytope and the objective function is linear, the support of $A^*$ contains only extreme points of the polytope: a finite set.

\section{Basic Measures and Relations in Information Theory}
 Before proceeding, we will define several common information measures -- entropy, mutual information, and Kullback-Leibler divergence --- and state several facts about these measures that will be referenced in our analysis. When all random variables are discrete, each of these facts is shown in chapter 2 of \citet{cover2012elements}. A treatment that applies to general random variables is provided in chapter 5 of \citet{gray2011entropy}. 

Before proceeding, we will define some simple shorthand notation. Let $P(X) = \Prob( X \in \cdot)$ denote the distribution function of random variable $X$. Similarly, define $P(X|Y) = \Prob( X \in \cdot |Y)$ and $P(X|Y=y) = \Prob( X \in \cdot |Y=y)$.

Throughout this section, we will fix random variables $X$, $Y$, and $Z$ that are defined on a joint probability space. We will allow $Y$ and $Z$ to be general random variables, but will restrict $X$ to be supported on a finite set $\mathcal{X}$. This is sufficient for our purposes, as we will typically apply these relations when $X$ is $A^*$, and is useful, in part, because the entropy of a general random variable can be infinite.

The {\it Shannon entropy} of $X$ is defined as 
\begin{equation*}
H(X) = - \sum_{x\in\mathcal{X}} \Prob(X=x) \log \Prob(X=x) .
\end{equation*}
The first fact establishes uniform bounds on the entropy of a probability distribution.
\begin{fact}\label{fact: entropy bounds}
$0 \leq H(X) \leq \log(|\mathcal{X}|)$. 
\end{fact}
If $Y$ is a discrete random variable, the entropy of $X$ conditional on $Y=y$ is 
\[ 
H(X|Y=y ) = \sum_{x \in \mathcal{X}} \Prob\left(X=x | Y=y \right)\log \Prob(X=x| Y=y)  
\]
and the conditional entropy is of $X$ given $Y$ is
\[ 
H(X|Y)= \sum_{y} H(X|Y=y)\mathbb{P}(Y=y).
\] 
For a general random variable $Y$, the conditional entropy of $X$ given $Y$ is, 
\[
H(X|Y ) = \E_{Y}\left[- \sum_{x \in \mathcal{X}} \Prob\left(X=x | Y\right)\log \Prob(X=x| Y) \right],
\]
where this expectation is taken over the marginal distribution of $Y$. For two probability measures $P$ and $Q$, if $P$ is absolutely continuous with respect to $Q$, the {\it Kullback--Leibler divergence} between them is
\begin{equation}
D(P || Q)= \intop\log \left( \frac{dP}{dQ} \right)dP 
\end{equation} 
where $\frac{dP}{dQ}$ is the Radon--Nikodym derivative of $P$ with respect to $Q$. This is the expected value under $P$ of the log-likelihood ratio between $P$ and $Q$, and is a measure of how different $P$ and $Q$ are. 
The next fact establishes the non-negativity of Kullback--Leibler divergence.
\begin{fact}\label{fact: nonnegative} (Gibbs' inequality) 
For any probability distributions $P$ and $Q$ such that $P$ is absolutely continuous with respect to $Q$, $D\left( P || Q\right) \geq 0$ with equality if and only if $P=Q$ $P$--almost everywhere.
\end{fact}

 The {\it mutual information} between $X$ and $Y$
\begin{equation}\label{eq: first mutual information definition}
I(X ; Y) = D\left( P( X, Y)   \, || \, P\left( X    \right) P\left( Y   \right)  \right) 
\end{equation}
is the Kullback--Leibler divergence between the joint distribution of $X$ and $Y$ and the product of the marginal distributions.  From the definition, it's clear that $I(X; Y)=I(X; A)$, and Gibbs' inequality implies that $I(X; Y)\geq 0$ and $I(X; Y)=0$ when $X$ and $Y$ are independent. 

The next fact, which is Lemma 5.5.6 of \citet{gray2011entropy},  states that the mutual information between $X$ and $Y$ is the expected reduction in the entropy of the posterior distribution of $X$ due to observing $Y$. 
\begin{fact}\label{fact: mutual information to entropy} 
(Entropy reduction form of mutual information) 
\[I\left( X ; Y \right)  = H(X)-H(X|Y)\] 
\end{fact}
The mutual information between $X$ and $Y$, conditional on a third random variable $Z$ is
\[ 
I( X; Y | Z) =  H(X|Z) -H(X| Y, Z),
\]
the expected additional reduction in entropy due to observing $Y$ given that $Z$ is also observed. This definition is also a natural generalization of the one given in \eqref{eq: first mutual information definition}, since 
\[
I( X; Y | Z) = \E_{Z}\left[ D\left( P\left((X,Y ) | Z  \right) \, ||\, P\left( X | Z  \right)P\left(X | Z  \right)  \right)\right].
\]
The next fact shows that conditioning on a random variable $Z$ that is independent of $X$ and $Y$ does not affect mutual information. 
\begin{fact}\label{fact: conditional MI under independence} If $Z$ is jointly independent of $X$ and $Y$, then
$I(X; Y | Z)  = I(X; Y)$. 
\end{fact}
The mutual information between a random variable $X$ and a collection of random variables $(Z_1,...,Z_T)$ can be expressed elegantly using the following ``chain rule.''  
\begin{fact}\label{fact: chain rule}(Chain Rule of Mutual Information)
\[
I(X; (Z_1,...Z_T)) = I\left( X; Z_1 \right)+I\left( X; Z_2 | \, Z_{1}  \right)+...+I\left( X; Z_T | \, Z_{1},...,Z_{T}  \right).
\]
\end{fact}
We provide some details related to the derivation of Fact \ref{fact: mutual information to KL} in the appendix. 
\begin{fact}\label{fact: mutual information to KL} (KL divergence form of mutual information) 
\begin{eqnarray*}
I\left( X ; Y \right)  &=& \E_{X}\left[ D\left(P(Y | X) \,\, || \,\, P(Y) \right)  \right] \\
&=& \sum_{x\in \mathcal{X}}\Prob(X=x) D \left( P(Y|X=x) \, ||\, P(Y ) \right)
\end{eqnarray*}
\end{fact}
While Facts \ref{fact: entropy bounds} and \ref{fact: mutual information to KL}, are standard properties of mutual information, it's worth highlighting their surprising power. It's useful to think $X$ as being $A^*$, the optimal action, and $Y$ as being $Y_{t,a}$, the observation  when selecting some action $a$. Then, combining these properties, we see that the next observation $Y$ is expected to greatly reduce uncertainty about the optimal action $A^*$ if and only if the distribution of $Y$ varies greatly depending on the realization of $A^*$, in the sense that $D \left( P(Y  | A^*=a^*) \, ||\, P(Y ) \right)$ is large on average. This fact is crucial to our analysis. 

One implication of the next fact is that the expected reduction in entropy from observing the outcome $Y_{t,a}$ is always at least as large as that from observing the reward $R(Y_{t,a})$. 
\begin{fact}(Weak Version of the Data Processing Inequality) If $Z=f(Y)$ for a deterministic function $f$, then $I(X; Y) \geq I(X; Z)$. If $f$ is invertible, so $Y$ is also a deterministic function of $Z$, then $I(X; Y) = I(X; Z)$. 
\end{fact}

We close this section by stating a fact that guarantees $D \left( P(Y|X=x) \, ||\, P(Y ) \right)$ is well defined. It follows from a general property of conditional probability: for any random variable $Z$ and event $E \subset \Omega$, if  $\Prob\left( E  \right) = 0$ then $\Prob\left( E| Z\right)=0$ almost surely. 
\begin{fact}
For any $x \in \mathcal{X}$ with $\Prob(X=x)>0$, $P(Y| X=x) $ is absolutely continuous with respect to $P(Y)$. 
\end{fact}

\subsection{Notation under posterior distributions}\label{subsec: posterior notation}
As shorthand, we let
\[ 
\Prob_{t}(\cdot) = \Prob(\cdot | \hist) = \Prob( \cdot |\,   A_1, Y_{1, A_1},...,A_{t-1}, Y_{t-1, A_{t-1}} )
\] 
and $\E_{t}\left[ \cdot \right] = \E[ \cdot | \hist  ]$.  As before, we will define some simple shorthand notation for the distribution function of a random variable under $\Prob_{t}$. Let $P_{t}(X) = \Prob_{t}( X \in \cdot)$, $P_{t}(X|Y) = \Prob_{t}( X \in \cdot |Y)$ and $P_{t}(X|Y=y) = \Prob_{t}( X \in \cdot |Y=y)$.

The definitions of entropy and mutual information implicitly depend on some base measure over the sample space $\Omega$. We will define special notation to denote entropy and mutual information under the posterior measure $\Prob_{t}(\cdot)$. Define
\begin{eqnarray*}
H_{t}(X) &=& -\sum_{x\in \mathcal{X}}\Prob_{t}(X=x) \log \Prob_{t}(X=x)\\
H_{t}(X|Y) &=& \E_{t}\left[-\sum_{x\in \mathcal{X}}\Prob_{t}(X=x|Y) \log \Prob_{t}(X=x|Y) \right]\\
I_{t}(X; Y) &=& H_{t}(X) - H_{t}(X|Y).
\end{eqnarray*}
Because these quantities depend on the realizations of $A_1, Y_{1, A_1},...,A_{t-1}, Y_{t-1, A_{t-1}}$, they are random variables. By taking their expectation, we recover the standard definition of conditional entropy and conditional mutual information: 
\begin{eqnarray*}
\E [H_{t}(X)] &=& H(X| A_1, Y_{1, A_1},...,A_{t-1}, Y_{t-1, A_{t-1}}) \\
\E [I_{t}(X; Y)] &=& I\left(X; Y  | A_1, Y_{1, A_1},...,A_{t-1}, Y_{t-1, A_{t-1}}   \right).
\end{eqnarray*}




\section{Thompson Sampling} 
The Thompson sampling algorithm simply samples actions according to the posterior probability they are optimal. In particular, actions are chosen randomly at time $t$ according to the sampling distribution $\pi^{\rm TS}_{t} = \Prob(A_t = \cdot | \hist)$. By definition, this means that for each $a \in \A$,  $\Prob (A_{t}=a | \hist)=\Prob (A^*=a | \hist)$. This algorithm is sometimes called {\it probability matching} because the action selection distribution is {\it matched} to the posterior distribution of the optimal action. 

This conceptually elegant probability matching scheme often admits a surprisingly simple and efficient implementation.  Consider the case where $\mathcal{P}=\{p_{\theta}\}_{\theta \in \Theta }$ is some parametric family of distributions . The true outcome distribution $p^*$ corresponds to a particular random index $\theta^* \in \Theta$ in the sense that $p^* = p_{\theta^*}$ almost surely. Practical implementations of Thompson sampling typically use two simple steps at each time $t$ to randomly generate an action from the distribution $\alpha_{t}$. First, an index $\hat{\theta}_{t} \sim \Prob\left( \theta^* \in \cdot \vert \hist \right)$ is sampled from the posterior distribution of the true index $\theta^*$. Then, the algorithm selects the action $A_t \in \underset{a\in \A}{\arg \max} \, \mathbb{E}\left[ R(Y_{t,a}) | \theta^* = \hat{\theta}_{t}  \right]$ that would be optimal if the sampled parameter were actually the true parameter. We next provide an example of a Thompson sampling algorithm designed to address the problem of online linear optimization under bandit feedback.

\subsection{Example of Thompson Sampling}
 Suppose each action $a \in \A \subset \mathbb{R}^d$ is defined by a $d$-dimensional feature vector, and almost surely there exists $\theta^* \in \mathbb{R}^d$ such that for each $a\in \A$, $\underset{ y \sim p_{a}^*}{\E }\left[ R(y) \right] = a^T \theta^*$. Assume $\theta^*$ is drawn from a normal distribution $N( \mu_{0}, \Sigma_{0})$. When $a$ is selected, only the realized reward $Y_{t,a}= R(Y_{t,a})  \in \mathbb{R}$ is observed. For each action $a$, reward noise $R(Y_{t,a}) - \mathbb{E}\left[R(Y_{t,a} | p^* \right]$ follows a Gaussian distribution with known variance. One can show that, conditioned on the history of observed data $\hist$, $\theta^*$ remains normally distributed.  Algorithm \ref{alg:linearPosteriorSampling} provides an implementation of Thompson sampling for this problem. The expectations in step 3 can be computed efficiently via Kalman filtering. 

\begin{figure}[H]
\algsetup{indent=2em}
\begin{algorithm}[H]
\caption{
 Linear--Gaussian Thompson Sampling}
\label{alg:linearPosteriorSampling}
\begin{algorithmic}[1]
\STATE \textbf{Sample Model}: \\
$\hat{\theta}_t \sim N(\mu_{t-1}, \Sigma_{t-1})$
\STATE \textbf{Select Action}: \protect\\
$A_{t}\in\arg\max_{a\in\A} \langle a, \hat{\theta}_{t}\rangle$ \protect\\
\STATE \textbf{Update Statistics}:  \\
$\mu_t \leftarrow \mathbb{E}[\theta^* | \hist ]$ \protect\\
$\Sigma_t \leftarrow \mathbb{E}[(\theta^*-\mu_t)(\theta^*-\mu_t)^\top | \hist ]$ \protect\\
\STATE \textbf{Increment $t$ and Goto Step 1}\\
\end{algorithmic}
\end{algorithm}
\end{figure}

Algorithm \ref{alg:linearPosteriorSampling} is efficient as long as the linear objective $\langle a, \hat{\theta}_{t}\rangle$ can be maximized efficiently over the action set $\A$. For this reason, the algorithm is implementable in important cases where other popular approaches, like the $\text{ConfidenceBall}_2$ algorithm of \citet{dani2008stochastic},  are computationally intractable.   Because the posterior distribution of $\theta^*$ has a closed form, Algorithm \ref{alg:linearPosteriorSampling} is particularly efficient. Even when the posterior distribution is complex, however, one can often generate samples from this distribution using Markov chain Monte Carlo algorithms, enabling efficient implementations of Thompson sampling. A more detailed discussion of the strengths and potential advantages of Thompson sampling can be found in earlier work \citep{scott2010modern, chapelle2011empirical, russo2013, gopalan2014thompson}.

\section{The Information Ratio and a General Regret Bound} 
Our analysis will relate the expected regret of Thompson sampling of Thompson sampling to its expected information gain: the expected reduction in the entropy of the posterior distribution of $A^*$. The relationship between these quantities is characterized by what we call the {\it information ratio}, 
\[
\Gamma_{t}  := \frac{\E_{t} \left[ R(Y_{t,A^*}) -R(Y_{t,A_t} )\right]^2 }{I_{t}\left( A^* ; (A_t, Y_{t, A_t})  \right)}
\]
which is the ratio between the square of expected regret and information gain in period $t$. Recall that, as described in Subsection \ref{subsec: posterior notation}, the subscript $t$ on $\E_{t}$ and $I_{t}$ indicates that these quantities are evaluated under the posterior measure $\Prob_{t}(\cdot)= \Prob(\cdot | \hist)$. 

Notice that if the information ratio is small, Thompson sampling can only incur large regret when it is expected to gain a lot of information about which action is optimal. This suggests its expected regret is bounded in terms of the maximum amount of information any algorithm could expect to acquire, which is at most the entropy of the prior distribution of the optimal action. Our next result shows this formally. We provide a general upper bound on the expected regret of Thompson sampling that depends on the time horizon $T$, the entropy of the prior distribution of $A^*$, $H(A^*)$, and any worst--case upper bound on the information ratio $\Gamma_{t}$. In the next section, we will provide bounds on $\Gamma_{t}$ for some of the most widely studied classes of online optimization problems.

\begin{proposition}\label{prop: regret bound}
For any $T\in \mathbb{N}$, if $\Gamma_t \leq \overline{\Gamma}$ almost surely for each $t \in \{1,..,T\}$,  
$$\E \left[{\rm Regret}(T, \pi^{\rm TS})  \right] \leq  \sqrt{\overline{\Gamma} H(\alpha_1) T}.$$
\end{proposition}
\begin{proof} Recall that $\E_{t}[\cdot] = \E[\cdot | \hist]$ and we use $I_{t}$ to denote mutual information evaluated under the base measure $\Prob_{t}$. Then,
\begin{eqnarray*}
\E \left[{\rm Regret}(T, \pi^{\rm TS})  \right] \overset{(a)}{=} \mathbb{E} \sum_{t=1}^{T} \E_{t} \left[ R(Y_{t, A^*}) -R(Y_{t,A_t} )\right]  &=& \mathbb{E}\sum_{t=1}^{T} \sqrt{\Gamma_{t} I_{t}\left( A^* ; (A_t, Y_{t, A_t})  \right)} \\ 
&\leq& \sqrt{\overline{\Gamma}}\left( \mathbb{E} \sum_{t=1}^{T} \sqrt{I_{t}\left( A^* ; (A_t, Y_{t, A_t})  \right)} \right) \\
&\overset{(b)}{\leq}& \sqrt{\overline{\Gamma}T \mathbb{E} \sum_{t=1}^{T} I_{t}\left( A^* ; (A_t, Y_{t, A_t})  \right)},
\end{eqnarray*}
where (a) follows from the tower property of conditional expectation, and (b) follows from the Cauchy-Schwartz inequality. We complete the proof by showing that expected information gain cannot exceed the entropy of the prior distribution. For the remainder of this proof, let $Z_t = (A_t, Y_{t,A_t})$. Then, as discussed in Subsection \ref{subsec: posterior notation},
\[
 \mathbb{E}\left[ I_{t}\left( A^* ;Z_t \right) \right] =  I\left( A^* ; Z_t | Z_1,...,Z_{t-1}  \right),
\]
and therefore
\begin{eqnarray*}
\mathbb{E}\sum_{t=1}^{T} I_{t}\left( A^* ; Z_t  \right) =  \sum_{t=1}^{T} I\left( A^* ; Z_t | Z_1,...,Z_{t-1}  \right)
&\overset{(c)}{=}&  I\left(A^* \, ;\,  Z_1,...Z_{T}    \right) \\ 
&=& H(A^*) - H(A^* | Z_1,...Z_{T}) \\
&\overset{(d)}{\leq}&  H(A^*),
\end{eqnarray*}
where (c) follows from the chain rule for mutual information (Fact \ref{fact: chain rule}), and (d) follows from the non-negativity of entropy (Fact \ref{fact: entropy bounds}). 
\end{proof}


\section{Bounding the Information Ratio}
This section establishes upper bounds on the information ratio in several important settings. This yields explicit regret bounds when combined with Proposition \ref{prop: regret bound}, and also helps to clarify the role the information ratio plays in our results: it roughly captures the extent to which sampling some actions allows the decision maker to make inferences about {\it different} actions. In the worst case, the information ratio depends on the number of actions, reflecting the fact that actions could provide no information about others. For problems with full information, the information ratio is bounded by a numerical constant, reflecting that sampling one action perfectly reveals the rewards that would have been earned by selecting any other action. The problems of  online linear optimization under ``bandit feedback'' and under ``semi--bandit feedback'' lie between these two extremes, and the information ratio provides a natural measure of  each problem's information structure. In each case, our bounds reflect that Thompson sampling is able to automatically exploit this structure. 

For each problem setting, we will compare our upper bounds on expected regret with known lower bounds. Some of these lower bounds were developed and stated in an adversarial framework, but were proved using the {\it probabilistic method}; authors fixed a family of distributions $\mathcal{P}$ and an initial distribution over $p^*$ and lower bounded the expected regret under this environment of any algorithm. This provides lower bounds on $\inf_{\pi} \E \left[{\rm Regret}(T, \pi) \right]$ in our framework\, at least for particular problem instances. Unfortunately, we are not aware of any general prior-dependent lower bounds, and this remains an important direction for the field.

Our bounds on the information ratio $\Gamma_{t}$ will hold at any time $t$ and under an posterior measure $\Prob_{t}$. To simplify notation, in our proofs we will omit the subscript $t$ from $\E_{t}, \Prob_{t}, P_t, A_t, Y_t, H_t,$ and $I_t$.

\subsection{An Alternative Representation of the Information Ratio}
Recall that the information ratio is defined to be  
\[ \frac{\E \left[ R(Y_{A^*}) -R(Y_{A} )\right]^2}{I\left(A^* ; (A, Y_{A})  \right)} .\]
The following proposition expresses the information ratio of Thompson sampling in a form that facilitates further analysis. 
The proof uses that Thompson sampling matches the action selection distribution to the posterior distribution of the optimal action, in the sense that $\Prob\left(A^*=a \right)=\Prob\left(A=a \right)$.
\begin{proposition}\label{prop: rewrite regret}
\begin{eqnarray*}
I\left(A^* ; (A, Y_{A})  \right)&=& \sum_{a\in \A} \Prob(A=a)I(A^*; Y_a)  \\
&=& \sum_{a, a^*\in \A}\Prob(A^*=a)\Prob(A^*=a^*)\left[ D\left( P(Y_{a} | A^*=a^*) \, || \,  P(Y_{a})  \right) \right].
\end{eqnarray*}
and
\[ \E \left[ R(Y_{A^*}) -R(Y_{A} )\right] = \sum_{a\in \A} \Prob(A^*=a) \left( \E\left[R(Y_a) | A^*=a \right] - \E[ R(Y_a)]\right. \]
\end{proposition}
The numerator of the information ratio measures the average difference between rewards generated from $P(Y_a)$, the posterior predictive distribution at $a$, and $P(Y_a \vert A^*=a)$, the posterior predictive distribution at $a$ conditioned on $a$ being the optimal action. It roughly captures how much  knowing that the {\it selected action is optimal} influences the expected reward observed. The denominator measures how much, on average, knowing {\it which action is optimal} changes the observations at the selected action. Intuitively, the information ratio tends to be small when knowing which action is optimal significantly influences the anticipated observations at many other actions. 

It's worth pointing out that this form of the information ratio bears a superficial resemblance to  fundamental complexity terms in the multi-armed bandit literature. The  results of \citet{lai1985asymptotically} and \citet{agrawal1989asymptotically} show the optimal asymptotic growth rate of regret is characterized by a ratio where the numerator depends on the difference between means of the reward distributions and the denominator depends on Kullback--Leibler divergences.
\begin{proof} 
Both proofs will use that the action $A$ is selected based on past observations and independent random noise. Therefore, conditioned on the history, $A$  is jointly independent of $A^*$ and the outcome vector $Y\equiv (Y_{a})_{a\in \A}$. 
\begin{eqnarray*}
I(A^* ; (A, Y_A)) &\overset{(a)}{=}& I(A^*; A) + I(A^* ; Y_A | A)  \\
&\overset{(b)}{=}&I(A^* ; Y_A | A)  \\
&=& \sum_{a\in \A}\Prob(A=a)I(A^* ; Y_A| A=a)\\
&\overset{(c)}{=}& \sum_{a\in \A}\Prob(A=a)I(A^* ; Y_a)\\
&\overset{(d)}{=}& \sum_{a\in \A}\Prob(A=a)\left(\sum_{a^* \in \A}\Prob(A^*=a^*) D\left(P(Y_a| A^*=a^*) \, || \, P(Y_a)   \right)\right)\\
&=&\sum_{a, a^*\in \A}\Prob(A^*=a)\Prob(A^*=a^*)\left[ D\left( P(Y_{a} | A^*=a^*) \, || \,  P(Y_{a})  \right) \right], 
\end{eqnarray*}
where (a) follows from the chain rule for mutual information (Fact \ref{fact: chain rule}), (b) uses that $A$ and $A^*$ are independent and the mutual information between independent random variables is zero (Fact \ref{fact: conditional MI under independence}), (c)  uses Fact \ref{fact: conditional MI under independence} and that $A$ is jointly independent of $Y$ and $A^*$, and equality (d) uses Fact \ref{fact: mutual information to KL}. Now, the numerator can be rewritten as, 
\begin{eqnarray*}
\E \left[ R(Y_{A^*}) -R(Y_{A} )\right] &=& \sum_{a\in \A} \Prob(A^*=a) \E\left[R(Y_a) | A^*=a \right] - \sum_{a\in \A}\Prob(A=a)\E[ R(Y_a) | A=a] \\
&=& \sum_{a\in \A} \Prob(A^*=a)\left( \E\left[R(Y_a) | A^*=a \right] - \E[ R(Y_a)]\right),
\end{eqnarray*}
where the second equality uses that $\Prob(A=a)=\Prob(A^*=a)$ by the definition of Thompson sampling, and that $Y$ is independent of the chosen action $A$. 
\end{proof}

\subsection{Preliminaries}  
Here we state two basic facts that are used in bounding the information ratio. Proofs of both results are provided in the appendix for completeness. 

The first fact lower bounds the Kullback--Leibler divergence between two bounded random variables in terms of the difference between their means. It follows trivially from an application of pinsker's inequality. 
\begin{fact}\label{fact: DME to DKL} For any distributions $P$ and $Q$ such that that $P$ is absolutely continuous with respect to $Q$,  any random variable $X: \Omega \rightarrow \mathcal{X}$ and any $g:\mathcal{X}\rightarrow \mathbb{R}$ such that $\sup g  - \inf g \leq 1$, 
$$\E_{P} \left[ g(X) \right] - \E_{Q} \left[ g(X) \right]  \leq \sqrt{\frac{1}{2} D \left( P || Q \right) },$$
where $\E_{P}$ and $\E_{Q}$ denote the expectation operators under $P$ and $Q$. \end{fact}
Because of Assumption \ref{assum: bounded rewards}, this fact shows 
\[
\E \left[ R(Y_{a}) | A^*=a^* \right] - \E\left[R(Y_{a})  \right] \leq \sqrt{\frac{1}{2}D\left(P(Y_{a} | A^* = a^*)   \,\, || \,\, P(Y_{a})  \right)}.
\]

By the Cauchy--Schwartz inequality, for any vector $x \in \mathbb{R}^n$, $\sum_{i}x_{i} \leq \sqrt{n} \| x\|_{2}$.   The next fact provides an analogous result for matrices. For any rank $r$ matrix $M \in \mathbb{R}^{n\times n}$ with singular values $\sigma_{1},...,\sigma_{r}$,  let
\begin{eqnarray*}
\| M\|_{*} : = \sum_{i=1}^{r} \sigma_{i}, \hspace{6pt}&\hspace{6pt}  \| M \|_F := \sqrt{ \sum_{k=1}^{m} \sum_{j=1}^{n} M_{i,j}^2} = \sqrt{\sum_{i=1}^{r} \sigma_{i}^2},  \hspace{6pt}&\hspace{6pt}  {\rm Trace}(M):=\sum_{i=1}^{n} M_{ii},
\end{eqnarray*}
denote respectively the Nuclear norm, Frobenius norm and trace of $M$. 

\begin{fact}\label{fact: trace frobenius inequality}
For any matrix $M \in \mathbb{R}^{k\times k}$, $${\rm Trace}\left( M \right) \leq \sqrt{{\rm Rank}(M)}\| M\|_{\rm F}.$$
\end{fact}

\subsection{Worst Case Bound}
The next proposition provides a bound on the information ratio that holds whenever rewards are bounded, but that has an explicit dependence on the number of actions. This scaling cannot be improved in general, but we  go on to show tighter bounds are possible for problems with different {\it information structures}. 
\begin{proposition}\label{prop: worst case bound}
For any $t \in \mathbb{N}$, $\Gamma_{t} \leq |\A|/2$ almost surely. 
\end{proposition}
\begin{proof}  We bound the numerator of the information ratio by $|\A|/2$ times its denominator: 
\begin{eqnarray*}
\E \left[ R(Y_{A^*}) -R(Y_{A} )\right]^2 &\overset{(a)}{=}&  \left(\sum_{a\in \A} \Prob(A^*=a) \left( \E\left[R(Y_a) | A^*=a \right] - \E[ R(Y_a)]\right) \right)^2 \\
&\overset{(b)}{\leq}& |\A|  \sum_{a\in \A} \Prob(A^*=a)^2 \left( \E\left[R(Y_a) | A^*=a \right] - \E[ R(Y_a)]\right)^2    \\
&\leq& |\A|  \sum_{a, a^* \in \A} \Prob(A^*=a)\Prob(A^*=a^*) \left( \E\left[R(Y_a) | A^*=a^* \right] - \E[ R(Y_a)]\right)^2   \\
&\overset{(c)}{\leq}& \frac{|\A|}{2} \sum_{a, a^* \in \A} \Prob(A^*=a)\Prob(A^*=a^*) D\left(P(Y_{a} | A^* = a^*)   \,\, || \,\, P(Y_{a})  \right)  \\
&\overset{(d)}{=}& \frac{|\A| I( A^* ; (A, Y))}{2}
\end{eqnarray*}
where (b) follows from the Cauchy--Schwarz inequality, (c) follows from Fact \ref{fact: DME to DKL}, and (a) and (d) follow from Proposition \ref{prop: rewrite regret}. 
\end{proof}
Combining Proposition \ref{prop: worst case bound} with Proposition \ref{prop: regret bound} shows that $\E \left[{\rm Regret}(T, \pi^{\rm TS})  \right] \leq \sqrt{\frac{1}{2} |\A| H(A^*)T}$. \citet{bubeck2013prior} show $\E \left[{\rm Regret}(T, \pi^{\rm TS})  \right] \leq 14\sqrt{ |\A| T}$ and that this bound is order optimal, in the sense that for any time horizon $T$ and number of actions $|\A|$ there exists a prior distribution over $p^*$ such that $\inf_{\pi} \E \left[{\rm Regret}(T, \pi) \right] \geq \frac{1}{20} \sqrt{|\A| T}$. 

\subsection{Full Information}
Our focus in this paper is on problems with {\it partial feedback}. For such problems, what the decision maker observes depends on the actions selected, which leads to a tension between exploration and exploitation. Problems with full information arise as an extreme point of our formulation where the outcome $Y_{t,a}$ is perfectly revealed by observing $Y_{t, \tilde{a}}$ for any $\tilde{a} \neq a$; what is learned does not depend on the selected action.  The next proposition shows that under full information, the information ratio is bounded by $1/2$.


\begin{proposition}\label{prop: full information}
Suppose for each $t\in \mathbb{N}$ there is a random variable $Z_{t}: \Omega \rightarrow \mathcal{Z}$ such that for each $a\in \A$,  $Y_{t,a}= \left( a, Z_{t}  \right)$. Then for all $t\in \mathbb{N}$,  $\Gamma_{t} \leq 1/2$ almost surely.  
\end{proposition}
\begin{proof}
As before we bound the numerator of $\Gamma_{t}$ by $\sqrt{1/2}$ times the denominator:
\begin{eqnarray*}
\E \left[ R(Y_{A^*}) -R(Y_{A} )\right] &\overset{(a)}{=}&  \sum_{a\in \A} \Prob(A^*=a) \left( \E\left[R(Y_a) | A^*=a \right] - \E[ R(Y_a)]\right) \\
&\overset{(b)}{\leq}& \sum_{a\in \A} \Prob(A^*=a) \sqrt{ \frac{1}{2} D\left( P(Y_{a} | A^*=a )\, || \, P(Y_a)  \right)  }\\ 
&\overset{(c)}{\leq}& \sqrt{ \frac{1}{2} \sum_{a\in \A} \Prob(A^*=a)  D\left( P(Y_{a} | A^*=a )\, || \, P(Y_a)  \right)  }\\ 
&\overset{(d)}{\leq}& \sqrt{ \frac{1}{2} \sum_{a, a^*\in \A} \Prob(A^*=a)\Prob(A^*=a^*)  D\left( P(Y_{a} | A^*=a^* )\, || \, P(Y_a)  \right)  }\\ 
&\overset{(e)}{=}& \sqrt{\frac{ I( A^* ; (A, Y))}{2} },
\end{eqnarray*}
where again (a) and (e) follow from Proposition \ref{prop: rewrite regret}, 
(b) follows from Fact \ref{fact: DME to DKL} and (c) follows from Jensen's inequality. Equality (d) follows because $D\left(P(Y_a | A^*=a^*) \, || \, P(Y_a)  \right)=D\left(P(Z| A^*=a^*  ) \, ||\, P(Z) \right)$ does not depend on the sampled action $a$ under full information. 
\end{proof}

Combining this result with Proposition \ref{prop: regret bound} shows $\E \left[{\rm Regret}(T, \pi^{\rm TS})  \right] \leq \sqrt{\frac{1}{2} H(A^*) T}$. Further, a worst--case bound on the entropy of $A^*$ shows that $\E \left[{\rm Regret}(T, \pi^{\rm TS})  \right] \leq \sqrt{\frac{1}{2} \log (|\A|) T}$. \citet{dani2007price} show this bound is order optimal, in the sense that there exists a class of online-linear prediction problems
under which $\inf_{\pi} \E \left[{\rm Regret}(T, \pi )  \right] = \Omega\left( \sqrt{\log (|\A|) T} \right)$. The bound here improves upon this worst case bound since $H(A^*)$ can be much smaller than $\log(|\A|)$ when the prior distribution is informative.

\subsection{Linear Optimization Under Bandit Feedback}\label{subsec: linear bandit}
The stochastic linear bandit problem has been widely studied \citep[e.g][]{dani2008stochastic, rusmevichientong2010linearly, abbasi2011improved} and is one of the most important examples of a multi-armed bandit problem with ``correlated arms.''  In this setting, each action is associated with a finite dimensional feature vector, and the mean reward generated by an action is the inner product between its known feature vector and some unknown parameter vector. Because of this structure, observations from taking one action allow the decision--maker to make inferences about other actions. The next proposition bounds the information ratio for such problems. Its proof is essentially a generalization of the proof of Proposition \ref{prop: worst case bound}.
\begin{proposition}\label{prop: linear}
If $\A \subset \mathbb{R}^d$ and for each $p\in \mathcal{P}$ there exists $\theta_p \in \mathbb{R}^d$ such that for all $a \in \A$
$$\underset{y\sim p_a}{\E}\left[ R(y)  \right] = a^T \theta_p,$$
 then for all $t\in \mathbb{N}$,  $\Gamma_{t} \leq d/2$ almost surely.  
\end{proposition}
\begin{proof}
Write $\A = \left\{a_{1},...,{a_K}\right\}$ and, to reduce notation, for the remainder of this proof let $\alpha_{i} = \Prob(A^*=a_i)$.  Define $M \in \mathbb{R}^{K \times K}$ by 
\[
M_{i,j} = \sqrt{\alpha_{i}\alpha_{j}} \left( \E[R(Y_{a_i}) | A^*=a_j] -\E[R(Y_{a_i})] \right),
\]
for all $i,j \in \{1,..,K \}$.  Then, by Proposition \ref{prop: rewrite regret},  
\[
\E\left[ R(Y_{A^*}) - R(Y_A)\right] = \sum_{i=1}^{K} \alpha_{i}\left( \E[R(Y_{a_i}) | A^*=a_i] -\E[R(Y_{a_i})] \right) =\rm{Trace}(M).
\]
Similarly, by Proposition \ref{prop: rewrite regret}, 
\begin{eqnarray*}
I(A^*; (A, Y_A))  &=& \sum_{i,j}\alpha_{i} \alpha_{j} D\left(P(Y_{a_i} | A^*=a_j) \, || \, P(Y_{a_i})  \right) \\
&\overset{(a)}{\geq} & 2 \sum_{i,j}\alpha_{i} \alpha_{j} \left( \E[R(Y_{a_{i}}) | A^*=a_j] -\E[R(Y_{a_i})]  \right)^2 \\
&=& 2\| M \|_{\rm F}^2,
\end{eqnarray*}
where inequality (a) uses Fact \ref{fact: DME to DKL}. This shows, by Fact \ref{fact: trace frobenius inequality}, that $$\frac{\E\left[ R(Y_{A^*}) - R(Y_A)\right]^2}{I(A^*; (A, Y_A)) } \leq \frac{ \rm{Trace}(M)^2}{2 \| M \|_{\rm F}^2} \leq \frac{{\rm Rank}(M)}{2} .$$ We now show ${\rm Rank}(M) \leq d$. Define 
\[
\mu = \mathbb{E}\left[ \theta_{p^*}  \right]  \hspace{30pt} \mu^j = \mathbb{E}\left[ \theta_{p^*}|  A^*=a_j \right]. 
\]
Then, by the linearity of the expectation operator, 
\[
M_{i,j} = \sqrt{\alpha_i \alpha_j}( (\mu^{j} - \mu)^T a_i)
\]
and therefore
$$M = \left[\begin{array}{c}
\sqrt{\alpha_1}\left(\mu^{1}-\mu \right)^{T}\\
\vdots\\
\vdots\\
\sqrt{\alpha_K)}\left(\mu^{k}-\mu \right)^{T}
\end{array}\right]\left[\begin{array}{cccc}
\sqrt{\alpha_{1}}a_{1} & \cdots & \cdots & \sqrt{\alpha_{K}}a_{K}\end{array}\right].
$$
This shows $M$ is the product of a $K$ by $d$ matrix and a $d$ by $K$ matrix, and hence has rank at most $d$. 
\end{proof}
This result shows that $\E \left[{\rm Regret}(T, \pi^{\rm TS})  \right] \leq \sqrt{\frac{1}{2} H(A^*)d T} \leq \sqrt{\frac{1}{2} \log(|\A|)d T}$ for linear bandit problems. Again, \citet{dani2007price} show this bound is order optimal, in the sense that for any time horizon $T$ and dimension $d$ if the actions set is  $\A = \{0,1 \}^d$, there exists a prior distribution over $p^*$ such that 
$\inf_{\pi} \E \left[{\rm Regret}(T, \pi )  \right] \geq c_{0} \sqrt{\log (|\A|) d T}$ where $c_{0}$ is a constant the is independent of $d$ and $T$. The bound here improves upon this worst case bound since $H(A^*)$ can be much smaller than $\log(|\A|)$ when the prior distribution is informative. 

\subsection{Combinatorial Action Sets and ``Semi--Bandit'' Feedback}\label{subsec: semi-bandit}
To motivate the information structure studied here, consider a simple resource allocation problem. There are $d$ possible projects, but the decision--maker can allocate resources to at most $m \leq d$ of them at a time. At time $t$, project $i \in \{1,..,d\}$ yields a random reward $\theta_{t,i}$, and the reward from selecting a subset of projects $a \in \A \subset \left\{ a' \subset \{0,1,...,d \} :  |a' | \leq m \right\}$ is $m^{-1} \sum_{i \in \A} \theta_{t,i}$. In the linear bandit formulation of this problem, upon choosing a subset of projects $a$ the agent would only observe the overall reward $m^{-1} \sum_{i \in a} \theta_{t,i}$. It may be natural instead to assume that the outcome of each selected project $(\theta_{t,i} : i \in a)$ is observed. This type of observation structure is sometimes called ``semi--bandit'' feedback \citep{audibert2013regret}.

A naive application of Proposition \ref{prop: linear} to address this problem would show $\Gamma_{t} \leq d/2$. The next proposition shows that since the entire parameter vector $(\theta_{t,i} : i \in a)$ is observed upon selecting action $a$,  we can provide an improved bound on the information ratio. The proof of the proposition is provided in the appendix. 

\begin{proposition}\label{prop: semi bandit}
Suppose $\A \subset \left\{ a \subset \{0,1,...,d \} :  |a | \leq m \right\}$, and that there are random variables $(\theta_{t,i}: t \in \mathbb{N}, i \in \{1,...,d  \})$ such that
\begin{eqnarray*}
Y_{t,a} = \left( \theta_{t,i} : i \in a \right) &\text{and}& R\left(Y_{t,a} \right) =\frac{1}{m} \sum_{i\in a }\theta_{t,i}.
\end{eqnarray*}
Assume that the random variables $\{\theta_{t,i} : i \in \{1,...,d \}\}$ are independent conditioned on $\hist$ and   $\theta_{t,i} \in [\frac{-1}{2},\frac{1}{2}]$ almost surely for each $(t,i)$. Then 
for all $t\in \mathbb{N}$,  $\Gamma_{t} \leq \frac{d}{2m^2}$ almost surely.  
\end{proposition}  
In this problem, there are as many as $\binom{d}{m}$ actions, but because Thompson sampling exploits the structure relating actions to one another, its regret is only polynomial in $m$ and $d$. In particular, combining Proposition \ref{prop: semi bandit} with Proposition \ref{prop: regret bound} shows $\E \left[{\rm Regret}(T, \pi^{\rm TS})  \right] \leq \frac{1}{m}\sqrt{d H(A^*)T }$. Since $H(A^*) \leq \log |\A| = O(m \log (\frac{d}{m}))$ this also yields a bound of order $\sqrt{\frac{d}{m} \log\left(\frac{d}{m}\right)T}$. As shown by \citet{audibert2013regret}, the lower bound\footnote{In their formulation, the reward from selecting action $a$ is $\sum_{i\in a }\theta_{t,i},$ which is $m$ times larger than in our formulation. The lower bound stated in their paper is therefore of order $\sqrt{mdT}$. They don't provide a complete proof of their result, but note that it follows from standard lower bounds in the bandit literature. In the proof of Theorem 5 in that paper, they construct an example in which the decision maker plays $m$ bandit games in parallel, each with $d/m$ actions. Using that example, and the standard bandit lower bound (see Theorem  3.5 of \citet{bubeck2012regret}), the agent's regret from each component must be at least $\sqrt{\frac{d}{m} T}$, and hence her overall expected regret is lower bounded by a term of order $m\sqrt{\frac{d}{m} T}=\sqrt{mdT}$.} for this problem is of order $\sqrt{\frac{d}{m}T}$, so our bound is order optimal up to a $\sqrt{\log(\frac{d}{m})}$ factor. It's also worth pointing that, although there may be as many as $\binom{d}{m}$ actions, the action selection step of Thompson sampling is computationally efficient whenever the offline decision problem $\max_{a\in\A} \theta^T a $ can be solved efficiently.

\section{Conclusion}

This paper has provided a new analysis of Thompson sampling based on tools from information theory.  As such, our analysis inherits the simplicity and elegance enjoyed by work in 
that field.  Further, our results apply to a much broader range of information structures than 
those studied in prior work on Thompson sampling.  Our analysis leads to regret bounds that highlight the benefits of soft knowledge, quantified in terms
of the entropy of the optimal-action distribution.
Such regret bounds yield insight into how future performance depends on past observations.  This is key to assessing the benefits of exploration,
and as such, to the design of more effective schemes that trade off between exploration and exploitation.  In forthcoming work, we leverage this insight 
to produce an algorithm that outperforms Thompson sampling.

While our focus has been on providing theoretical guarantees for Thompson sampling, we believe the techniques and quantities used in the analysis may be of broader interest. Our formulation and notation may be complex, but the proofs themselves essentially follow from  combining known relations in information theory with the tower property of conditional expectation, Jensen's inequality, and the Cauchy--Schwartz inequality.  In addition, the information theoretic view taken in this paper may provide a fresh perspective on this class of problems.

\acks{Daniel Russo is supported by a Burt and Deedee McMurty Stanford Graduate Fellowship. This work
was supported, in part, by the National Science Foundation under Grant CMMI-0968707. 
The authors would like to thank the anonymous reviewers for their helpful comments. }



\appendix

\section{Proof of Basic Facts}
\subsection{Proof of Fact \ref{fact: DME to DKL}}
This result is a consequence of {\it Pinsker's inequality}, (see Lemma 5.2.8 of \citet{gray2011entropy}) which states that 
\begin{equation}
\sqrt{\frac{1}{2}D(P  ||Q)}\geq \| P-Q \|_{\rm TV}
\end{equation}
where $\| P-Q \|_{\rm TV}$ is the total variation distance between $P$ and $Q$. When $\Omega$ is countable, 
$\| P-Q \|_{\rm TV} =\frac{1}{2} \sum_{\omega} |P(\omega)-Q(\omega)|$. More generally, if $P$ and $Q$ are both absolutely continuous with respect to some base measure $\mu$, then $\| P-Q \|_{\rm TV} =\frac{1}{2} \intop_{\Omega} |\frac{dP}{d\mu}-\frac{dQ}{d\mu}|d\mu$, where $\frac{dP}{d\mu}$ and $\frac{dQ}{d\mu}$ are the Radon--Nikodym derivatives of $P$ and $Q$ with respect to $\mu$. 
We now prove Fact \ref{fact: DME to DKL}. 
\begin{proof}
Choose a base measure $\mu$ so that $P$ and $Q$ are absolutely continuous with respect to $\mu$. This is always possible: since $P$ is absolutely continuous with respect to $Q$ by hypothesis, one could always choose this base measure to be $Q$. Let $f(\omega)=g(X(\omega)) - \inf_{\omega} g(X(\omega)) - 1/2$ so that $f: \Omega \rightarrow [-1/2, 1/2]$ and $f$ and $g(X)$ differ only by a constant. Then, 
\begin{eqnarray*}\sqrt{\frac{1}{2} D(P  ||Q)} \geq \frac{1}{2}\intop \left|\frac{dP}{d\mu}-\frac{dQ}{d\mu} \right| d\mu  &\geq &  \frac{1}{2} \intop \left|2 \left( \frac{dP}{d\mu}-\frac{dQ}{d\mu}\right)f  \right| d\mu  \\
&\geq&  \intop \left( \frac{dP}{d\mu}-\frac{dQ}{d\mu}\right)f d\mu \\
&=&  \intop f dP - \intop f dQ \\
&=& \E_{P}[g(X)] - \E_{Q}[g(X)],
\end{eqnarray*}
where the first inequality is Pinsker's inequality. 
\end{proof}

\subsection{Proof of Fact \ref{fact: trace frobenius inequality}}
\begin{proof}
Fix a rank $r$ matrix $M \in \mathbb{R}^{k\times k}$ with singular values $\sigma_{1}\geq ...\geq \sigma_{r}$.  By the Cauchy Shwartz inequality, 
\begin{equation}\label{eq: nuclear norm by Frobenius norm}
\| M  \|_* \,\, \overset{{\rm Def}}{=} \,\, \sum_{i=1}^{r} \sigma_{i} \leq \sqrt{r} \sqrt{\sum_{i=1}^{r} \sigma_{i}^2} \,\, \overset{{\rm Def}}{=} \,\, \sqrt{r} \| M \|_{F}.\end{equation} 
Now, we show
\begin{eqnarray*}
{\rm Trace}(M) = {\rm Trace}\left(\frac{1}{2}M+\frac{1}{2}M^T\right) \overset{(a)}{\leq} \left\| \frac{1}{2}M+\frac{1}{2}M^T \right\|_{*} & \overset{(b)}{\leq}& \frac{1}{2}\left\| M \right\|_{*}+\frac{1}{2}\left\| M^T \right\|_{*} \\
&\overset{(c)}{=}& \| M \|_{*} \overset{(d)}{\leq} \sqrt{r} \| M \|_F.
\end{eqnarray*}
Here (b) follows from the triangle inequality and the fact that norms are homogeneous of degree one. Inequality (c) uses that the singular values of $M$ and $M^T$ are the same, and inequality (d) follows from equation \eqref{eq: nuclear norm by Frobenius norm}. 

To justify inequality (a), we show that for any symmetric matrix $W$, $ {\rm Trace}(W) \leq \| W \|_{*}$. To see this, let $W$ be a rank $r$ matrix and let $\tilde{\sigma}_{1}\geq...\geq \tilde{\sigma}_{r}$ denote its singular values. Since $W$ is symmetric, it has  $r$ nonzero real valued eigenvalues $\lambda_{1},..,\lambda_{r}$. If these are sorted so that $|\lambda_{1}| \geq ...\geq |\lambda_{r}|$  then $\tilde{\sigma}_{i}=|\lambda_{i}|$ for each $i\in \{1,..,r \}. $\footnote{This fact is stated, for example, in Appendix A.5 of \citet{boyd2004convex}.} This shows  ${\rm Trace}(M) = \sum_{i=1}^{r} \lambda_{i} \leq  \sum_{i=1}^{r}  \tilde{\sigma_{i}}  = \|  M\|_{*}$. 
\end{proof}

\section{Proof of Proposition \ref{prop: semi bandit}}
The proof relies on the following lemma, which lower bounds the information gain due to selecting an action $a$. The proof is provided below, and relies on the chain--rule of Kullback--Leibler divergence. 
\begin{lem}\label{lem: semi bandit info gain}
Under the conditions of Proposition \ref{prop: semi bandit}, for any $a\in \A$, 
\begin{equation*}\label{eq: semi bandit info gain}
I(A^*; Y_a)  \geq 2\sum_{i\in a} \Prob(i \in A^*)  \left(  \E \left[    \theta_{i}    | i \in A^*  \right]-\E \left[    \theta_{i}\right] \right)^2
\end{equation*}
\end{lem} 
\begin{proof} 
\newcommand{\thetaja}{{\theta_{j<i}}}
In the proof of this lemma, for any $i \in a$ we let $\thetaja = (\theta_{j}: j<i, j\in a)$. Since $a \in \A$ is fixed throughout this proof, we do not display the dependence of $\thetaja$ on $a$. Recall that by Fact \ref{fact: mutual information to KL}, 
\[
I(A^*; Y_a) = \sum_{a^* \in \A} \Prob(A^*=a^*)  D\left( P(Y_a| A^*=a^*)  \, || \,  P(Y_a) \right).
\]
where here,
\begin{eqnarray*}
D\left( P(Y_a| A^*=a^*)  \, || \,  P(Y_a) \right) &=& D\left( P\left( ( \theta_{i}: i \in a)  | A^*=a^* \right) \,\, || \,\, P\left( ( \theta_{i}: i \in a)  \right) \right)\\
&\overset{(a)}{=}& \sum_{i\in a} \E\left[ D\left( P \left(    \theta_{i} | A^* =a^*, \thetaja \right) \,\,||\,\, P\left(    \theta_{i}  |  \thetaja \right)
\right)  \bigg\vert A^*=a^*   \right] \\
&\overset{(b)}{\geq}&  2 \sum_{i\in a} \E  \left[  \left( \E \left[    \theta_{i}    |  A^* =a^*, \thetaja \right]-\E \left[    \theta_{i}    |  \thetaja \right] \right)^2 \bigg| A^*=a^*  \right]\\
&\overset{(c)}{=}& 2\sum_{i\in a} \E  \left[  \left( \E \left[    \theta_{i}    | A^* =a^*, \thetaja \right]-\E \left[    \theta_{i} \right] \right)^2 \bigg| A^*=a^*  \right]\\
&\overset{(d)}{\geq}& 2\sum_{i\in a} \left( \E \left[    \theta_{i} | A^* =a^* \right]-\E \left[    \theta_{i} \right] \right)^2.
\end{eqnarray*}
Equality (a) follows from the chain rule of Kullback--Leibler divergence, (b) follows from Fact \ref{fact: DME to DKL}, (c) follows from the assumption that $(\theta_{i} : 1\leq i \leq d)$ are independent conditioned on any history of observations, and (d) follows from Jensen's inequality and the tower property of conditional expectation. Now, we can show
\begin{eqnarray*} 
 I(A^*; Y_a) 
&\geq&  2\sum_{i\in a} \sum_{a^* \in \A} \Prob(A^*=a^*)  \left( \E \left[    \theta_{i} | A^* =a^* \right]-\E \left[    \theta_{i} \right] \right)^2 \\
&=& 2\sum_{i\in a} \Prob(i \in A^*) \sum_{a^* \in \A} \frac{ \Prob(A^*=a^*) }{\Prob(i \in A^*)}  \left( \E \left[    \theta_{i}    |  A^* =a^* \right]-\E \left[    \theta_{i} \right] \right)^2 \\
&\geq& 2\sum_{i\in a} \Prob(i \in A^*) \sum_{a^*: i \in a^*} \frac{\Prob(A^*=a^*)}{\Prob(i \in A^*)}  \left( \E \left[    \theta_{i}    | A^* =a^* \right]-\E \left[    \theta_{i}  \right] \right)^2\\
&\geq& 2\sum_{i\in a} \Prob(i \in A^*) \sum_{a^*: i \in a^*} \Prob(A^*=a^*| i \in A^*)  \left( \E \left[    \theta_{i}    | A^* =a^* \right]-\E \left[    \theta_{i}  \right] \right)^2\\
&\overset{(e)}{\geq}& 2\sum_{i\in a} \Prob(i \in A^*)  \left(\sum_{a^*: i \in a^*} \Prob(A^*=a^*| i \in A^*) \left( \E \left[    \theta_{i}    | A^* =a^* \right]-\E \left[    \theta_{i} \right]\right) \right)^2\\
& \overset{(f)}{=}& 2\sum_{i\in a} \Prob(i \in A^*)  \left(  \E \left[    \theta_{i}    | i \in A^*  \right]-\E \left[    \theta_{i} \right] \right)^2,
\end{eqnarray*}
where (e) and (f) follow from Jensen's inequality and the tower property of conditional expectation, respectively. 
\end{proof}
\begin{lem}
Under the conditions of Proposition \ref{prop: semi bandit}, 
\[ 
I(A^*; (A, Y_A)) \geq 2\sum_{i=1}^{d} \Prob(i \in A^*)^2  \left(  \E \left[    \theta_{i} | i \in A^*  \right]-\E \left[    \theta_{i}  \right] \right)^2.
\]
\end{lem}
\begin{proof}
By Lemma \ref{lem: semi bandit info gain} and the tower property of conditional expectation, 
\begin{eqnarray*}
I(A^*; (A; Y_A)) &=& \Prob(A^*=a)I(A^* ; Y_a) \\
&\geq& 2\sum_{a\in \A}\Prob(A^*=a) \left[ \sum_{i\in a} \Prob(i \in A^*)  \left(  \E \left[    \theta_{i}    | i \in A^*  \right]-\E \left[    \theta_{i} \right] \right)^2 \right] \\
&=& 2\sum_{i=1}^{d} \sum_{a: i \in a} \Prob(A^* =a)  \left( \Prob(i \in A^*)  \left(  \E \left[    \theta_{i}  |  i \in A^*  \right]-\E \left[    \theta_{i} \right] \right)^2  \right)  \\
&=& 2\sum_{i=1}^{d} \Prob(i \in A^*)^2  \left(  \E \left[    \theta_{i} | i \in A^*  \right]-\E \left[    \theta_{i}  \right] \right)^2.
\end{eqnarray*}
\end{proof}
Now, we complete the proof of Proposition \ref{prop: semi bandit}. 
\begin{proof}
The proof establishes that the numerator of the information ratio is less than $d/2m^2$ times its denominator: 
\begin{eqnarray*}
m\E[R(Y_{A^*}) -R(Y_{A})] &=&\E \sum_{i\in A^*}\theta_i- \E \sum_{i\in A}\theta_i  \\
&=& \sum_{i=1}^{d} \Prob(i\in A^*) \E[\theta_i | i\in A^*] -  \sum_{i=1}^{d} \Prob(i\in A) \E[\theta_i | i\in A] \\
&=&\sum_{i=1}^{d} \Prob(i\in A^*) \left( \E[\theta_i | i\in A^*] - \E[\theta_i] \right)\\
&\leq& \sqrt{d} \sqrt{\sum_{i=1}^{d} \Prob(i\in A^*)^2 \left( \E[\theta_i | i\in A^*] - \E[\theta_i] \right)^2}\\
&\leq& \sqrt{\frac{d I(A^*; (A, Y_A))}{2}}.
\end{eqnarray*}
\end{proof}

\section{Proof of Fact \ref{fact: mutual information to KL}}
We could not find a reference that proves Fact \ref{fact: mutual information to KL} in a general setting, and will therefore provide a proof here. 

Consider random variables $X: \Omega \rightarrow \mathcal{X}$ and $Y: \Omega \rightarrow \Y$ where $\mathcal{X}$ is assumed to be a finite set but $Y$ is a general random variable. We show 
$$ I(X; Y) = \sum_{x\in \mathcal{X}} \Prob \left(X=x \right) D\left( \Prob \left( Y \in \cdot | X=x \right) \, || \,  \Prob \left( Y \in \cdot \right) \right).$$
When $Y$ has finite support this result follows easily by using that $\Prob(X = x, Y=y) = \Prob(X=x)\Prob(Y=y| X=x)$. For an extension to general random variables $Y$ we follow chapter 5 of \citet{gray2011entropy}. 

\newcommand{\Q}{\mathcal{Q}}
\newcommand{\YQ}{\left[Y\right]_{\mathcal{Q}}}
The extension follows by considering quantized versions of the random variable $Y$. For a finite partition $\Q = \{Q_{i} \subset \Y \}_{i \in I}$, we denote by  $Y_{\Q}$ the quantization of $Y$ defined so that  $Y_{\Q}(\omega) = i$ if and only if  $Y(\omega) \in Q_i$. 
As shown in \citet{gray2011entropy}, for two probability measures $P_1$ and $P_2$ on $\left(\Omega, \mathcal{F}  \right)$, 
\begin{equation}
D\left( P_{1}( Y  \in \cdot  ) ||P_{2}( Y \in \cdot  )   \right) = \sup_{\Q} D\left( P_{1}( Y_{\Q} \in \cdot  ) ||P_{2}( Y_{\Q} \in \cdot  )\right),
\end{equation}
where the supremum is taken over all finite partitions of $\Y$, and whenever $\overline{\Q}$ is a refinement of $\Q$,  
\begin{equation}
 D\left( P_{1}( Y_{\overline{\Q}} \in \cdot  ) ||P_{2}( Y_{\overline{\Q}} \in \cdot  )\right) \geq  D\left( P_{1}( Y_{\Q} \in \cdot  ) ||P_{2}( Y_{\Q} \in \cdot  )\right).
\end{equation}
Since $I(X; Y) = D\left(\Prob(X \in \cdot, Y\in \cdot) \, || \, \Prob( X\in \cdot) \Prob(Y\in \cdot)   \right)$ these properties also apply to $I(X; Y_{\Q})$. 
Now, for each $x \in \mathcal{X}$ we can introduce a sequence of successively refined partitions $\left( \Q_{n}^{x} \right)_{n\in \mathbb{N}}$ so that 
\begin{eqnarray*}
  D\left( \Prob \left(  Y \in \cdot | X=x \right) \, || \,  \Prob \left( Y \in \cdot \right) \right)&=&  \lim_{n \rightarrow \infty}  D\left( \Prob \left(  Y_{\Q_n^x} \in \cdot | X=x \right) \, || \,  \Prob \left( Y_{\Q_n^x} \in \cdot \right) \right) \\
&=&\sup_{\Q} D\left( \Prob \left(  Y_{\Q} \in \cdot | X=x \right) \, || \,  \Prob \left( Y_{\Q} \in \cdot \right) \right).
\end{eqnarray*}
Let the finite partition $\Q_n$ denote the refinement of the partitions $\{\Q_n^x \}_{x\in \mathcal{X}}$. That is,  for each $x$, every set in $Q_n^x$ can be written as the union of sets in $\Q_n$.  Then,  
\begin{eqnarray*}
I(X;  Y) &\geq& \underset{n\rightarrow \infty}{\lim} I(X; Y_{\Q_n}) \\
&=& \underset{n\rightarrow \infty}{\lim}\sum_{x\in \mathcal{X}} \Prob \left(X=x \right) D\left( \Prob \left(Y_{\Q_n} \in \cdot | X=x \right) \, || \,  \Prob \left( Y_{\Q_n} \in \cdot \right) \right)  \\
&\geq& \underset{n\rightarrow \infty}{\lim}\sum_{x\in \mathcal{X}} \Prob \left(X=x \right) D\left( \Prob \left(Y_{\Q_n^x} \in \cdot | X=x \right) \, || \,  \Prob \left( Y_{\Q_n^x} \in \cdot \right) \right)  \\
&=& \sum_{x\in \mathcal{X}} \Prob \left(X=x \right) D\left( \Prob \left(Y  \in \cdot | X=x \right) \, || \,  \Prob \left( Y \in \cdot \right) \right)  \\
&=& \sum_{x\in \mathcal{X}} \Prob \left(X=x \right) \sup_{\Q^x}  D \left( \Prob \left(Y_{\Q^x} \in \cdot | X=x \right) \, || \,  \Prob \left( Y_{\Q^x} \in \cdot \right) \right) \\
&\geq& \sup_{\Q} \sum_{x\in \mathcal{X}} \Prob \left(X=x \right)   D\left( \Prob \left(Y_{\Q} \in \cdot | X=x \right) \, || \,  \Prob \left( Y_{\Q} \in \cdot \right) \right) \\
&=& \sup_{\Q} I(X; Y_{\Q}) = I(X; Y),
\end{eqnarray*}
hence the inequalities above are equalities and our claim is established.

\section{Technical Extensions}

\subsection{Infinite Action Spaces}\label{sec: infinite action spaces}
In this section, we discuss how to extend our results to problems where the action space is infinite. For concreteness, we will focus on linear optimization problems. Specifically, throughout this section our focus is on problems  where $\A \subset \mathbb{R}^d$ and for each $p \in \mathcal{P}$ there is some $\theta_{p} \in \Theta \subset \mathbb{R}^d$ such that for all $a \in \A$
$$ \underset{y\sim p_{a}}{\E} \left[ R(y) \right] = a^T \theta_{p}.$$
We assume further that there is a fixed constant $c \in \mathbb{R}$ such that $\sup_{a\in\A} \| a \|_2 \leq c$ and $\sup_{\theta \in \Theta} \| \theta \|_2 \leq c$. Because $\A$ is compact, it can be extremely well approximated by a finite set. In particular, we can choose  a finite cover $\A_{\epsilon} \subset \A$ with $\log|A_{\epsilon}|= O(d\log(dc/\epsilon)$ such that for every $a \in \A$, $\min_{\tilde{a}\in \A_{\epsilon}} \| a- \tilde{a}\|_2 \leq \epsilon$. By the Cauchy-Schwartz inequality, this implies that $| \left(a-\tilde{a}\right)^T \theta_P | \leq c\epsilon$  for every $\theta_P$. 

We introduce a quantization $A^*_Q$ of the random variable $A^*$. $A_Q^*$ is supported only on the set $\A_{\epsilon}$, but satisfies $\| A^*(\omega)-A^*_Q(\omega) \|_2 \leq \epsilon$ for each $\omega \in \Omega$. 

Consider a variant of Thompson sampling that randomizes according to the distribution of $A_Q^*$. That is, for each $a\in \A_{\epsilon}$, $\Prob\left(A_{t}=a | \hist  \right)=\Prob\left(A^*_Q=a | \hist  \right)$. Then, our analysis shows, 
\begin{eqnarray*}
\E \sum_{t=1}^{T}\left[ R( Y_{t, A^*}) -R( Y_{t, A_t})  \right] &=& \E \sum_{t=1}^{T}\left[ R( Y_{t,A^*})-R( Y_{t,A_Q^*})  \right]+\E \sum_{t=1}^{T}\left[R( Y_{t, A_Q^*}) -R( Y_{t, A_{t}})  \right] \\
&\leq& \epsilon T + \sqrt{dT H\left(A_Q^* \right)}.
\end{eqnarray*}

This new bound depends on the time horizon $T$, the dimension $d$ of the linear model, the entropy of the prior distribution of the quantized random variable $A^*_Q$ and the discretization error $\epsilon T $. Since $H(A_Q^*) \leq \log\left(|\A_{\epsilon}| \right)$, by choosing $\epsilon=(cT)^{-1}$ and choosing a finite cover with $\log\left(|\A_{\epsilon}|\right) = O(d\log(dcT))$ we attain a regret bound of order 
$d\sqrt{T \log(dcT)}$. Note that for $d\leq T$, this bound is of order $d\sqrt{T \log(cT)}$, but for $d>T \geq 1$, we have a trivial regret bound of $T< \sqrt{dT}$. 

Here, for simplicity, we have considered a variant of Thompson sampling that randomizes according to the posterior distribution of the quantized random variable $A^*_Q$. However, with a somewhat more careful analysis, one can provide a similar result for an algorithm that randomizes according to the posterior distribution of $A^*$.

\subsection{Unbounded Noise} 
In this section we relax Assumption \ref{assum: bounded rewards}, which requires that reward distributions are uniformly bounded. This assumption is required for Fact \ref{fact: DME to DKL} to hold, but otherwise was never used in our analysis. Here, we show that an analogue to Fact \ref{fact: DME to DKL} holds in the more general setting where reward distributions are {\it sub--Gaussian}. This yields more general regret bounds. 

A random variable is sub--Gaussian if its moment generating function is dominated by that of a Gaussian random variable. Gaussian random variables are sub--Gaussian, as are real-valued random varaibles with bounded support. 
\begin{defn}
Fix a deterministic constant $\sigma \in \mathbb{R}$. A real--valued random variable $X$ is $\sigma$ sub--Gaussian if for all $\lambda \in \mathbb{R}$ 
$$\E \left[\exp\{ \lambda X \} \right]\leq \exp\{ \lambda^2 \sigma^2 /2\}.$$
More generally $X$ is $\sigma$ sub--Gaussian conditional on a sigma-algebra $\mathcal{G} \subset \mathcal{F}$ if for all $\lambda \in \mathbb{R}$, 
$$\E \left[\exp\{ \lambda X \} \vert \mathcal{G} \right]\leq \exp\{ \lambda^2 \sigma^2 /2\}$$ 
almost surely. 
\end{defn}
The next lemma extends Fact \ref{fact: DME to DKL} to sub-Gaussian random variables. 
\begin{lem}\label{lem: DME to DKL under sub--Gaussian noise}
Suppose there is a $\hist$ measurable random variable $\sigma$ such that for each $a\in\A$, $R(Y_{t,a}))$ is $\sigma$ sub-Gaussian conditional on $\hist$. Then for every $a, a^* \in \A$ 
\[
\E_{t}[R(Y_{t,a})| A^*=a^*] -\E_{t}[R(Y_{t,a})] \leq \sigma \sqrt{2D \left(P_{t}(Y_{t,a}|A*=a^*) \, || \, P_{t}(Y_{t,a})\right)}, 
\]
almost surely. 
\end{lem}
Using this Lemma in place of Fact \ref{fact: DME to DKL} in our analysis leads to the following corollary. 
\begin{cor}\label{cor: information ratio under sub--Gaussian noise}
Fix a deterministic constant $\sigma \in \mathbb{R}$. Suppose that conditioned on $\hist$,  $R(Y_{t,a})-\E\left[R(Y_{t,a} \vert \hist  \right] $ is $\sigma$ sub--Gaussian. Then
$$ \Gamma_{t} \leq 2 |\A|\sigma^2$$
almost surely for each $t\in \mathbb{N}$. Furthermore, $\Gamma_{t} \leq 2\sigma^2 $ under the conditions of Proposition \ref{prop: full information}, $\Gamma_{t} \leq 2d\sigma^2$ under the conditions of Proposition \ref{prop: linear}, and $\Gamma_{t} \leq \frac{2d\sigma^2}{m^2}$ under the conditions of Proposition \ref{prop: semi bandit}. 
\end{cor}
It's worth highlighting two cases under which the conditions of Corollary \ref{cor: information ratio under sub--Gaussian noise} are satisfied. The first is a setting with a Gaussian prior and Guassian reward noise. The second case is when under any model $p\in\mathcal{P}$ expected rewards lie in a bounded interval and reward noise is sub--Gaussian. 
\begin{enumerate}
\item Suppose that for each $a\in \A$,  $R(Y_{t,a})$ follows a Guassian distribution with variance less than $\sigma^2$, $Y_{t,a}=R(Y_{t,a})$ and, reward noise $R(Y_{t,a})-\E[R(Y_{t,a}) | p^*]$ is Gaussian with known variance. Then, the posterior predictive distribution of $R(Y_{t,a})$, $\Prob\left( R(Y_{t,a}) \in \cdot| \hist \right)$ is Gaussian with variance less than $\sigma^2$ for each $a\in\A$ and $t\in \mathbb{N}$. 
\item Fix constants $C\in \mathbb{R}$ and $\sigma \in \mathbb{R}$. Suppose that almost surely $\E \left[ R(Y_{t,a}) | p^*\right] \in \left[\frac{-C}{2}, \frac{C}{2}\right] $  and reward noise is $R(Y_{t,a})-\E \left[ R(Y_{t,a}) | p^*\right]$ is $\sigma$ sub--Gaussian. Then, conditioned on the history, $R(Y_{t,a})-\E \left[ R(Y_{t,a}) | \hist \right]$ is $\sqrt{C^2 + \sigma^2}$ sub-Gaussian almost surely. 
\end{enumerate}
The first case relies on the fact that $\E[R(Y_{t,a}) | p^*]$ is $C$--sub-Gaussain, as well as the following fact about sub-Gaussian random variables. 
\begin{fact}
Consider random variables $\left( X_{1},...,X_{K} \right)$. If for each $X_{i}$  there is a deterministic $\sigma_{i}$ such that $X_{i}$ is $\sigma_{i}$ sub--Gaussian conditional on $X_{1},..,X_{i-1}$, then $\sum_{i=1}^{K} X_{i}$ is 
$\sqrt{ \sum_{i=1}^{K} \sigma_{i}^2}$ sub-Gaussian. 
\end{fact}
\begin{proof} For any $\lambda \in \mathbb{R}$,
\begin{eqnarray*}
\E \left[ \exp\{ \lambda \sum_{i=1}^{K} X_{i}\}\right] = \E \left[ \prod_{i=1}^{K} \exp\{ \lambda X_{i}\}\right] &=& \E \left[ \prod_{i=1}^{K}\E \left[ \exp\{ \lambda X_{i}\}  \bigg\vert X_{1},..X_{i-1}\right] \right] \\
&\leq& \prod_{i=1}^{K} \exp\left\{ \frac{\lambda^2 \sigma_{i}^2}{2} \right\} = \exp\left\{ \frac{\lambda^2 (\sum_{i}\sigma_{i}^2)}{2}\right\}.
\end{eqnarray*}
\end{proof}
The proof of Lemma \ref{lem: DME to DKL under sub--Gaussian noise} relies on the following property of Kullback--Leibler divergence. 
\begin{fact}\label{fact: variational form} (Variational form of KL--Divergence given in Theorem 5.2.1 of \citet{gray2011entropy}) Fix two probability distributions $P$ and $Q$ such that $P$ is absolutely continuous with respect to $Q$. Then, 
$$D (P || Q) = \sup_{X} \left\{ \E_{P} \left[ X \right] - \log \E_Q\left[ \exp\{ X \}\right] \right\},$$ where $\E_{P}$ and $\E_{Q}$ denote the expectation operator under $P$ and $Q$ respectively, and the supremum is taken over all real valued random variables $X$ such that $\E_{P} \left[ X \right]$ is well defined and $\E_Q\left[ \exp\{ X \}\right]<\infty$. 
\end{fact}
We now show that Lemma \ref{lem: DME to DKL under sub--Gaussian noise} follows from the variational form of KL--Divergence. 
\begin{proof}
Define the random variable $X=R\left(Y_{t,a} \right)- \E_{t}\left[ R\left(Y_{t,a} \right)\right]$. Then, for abitrary $\lambda \in \mathbb{R}$, applying Fact \ref{fact: variational form} to $\lambda X$ yields
\begin{eqnarray*}
D \left(P_{t}\left( Y_{t,a} | A^*=a^* \right) \, || \, P_{t}(Y_{t,a})\right) &\geq& \lambda \E_{t}\left[X | A^*=a^* \right] - \log \E_{t} \left[ \exp\{\lambda X \} \right]\\
&= & \lambda \left( \E_{t}[R\left(Y_{t,a} \right)| A^*=a^*]- \E_{t}\left[ R\left(Y_{t,a} \right)\right] \right)- \log \E \left[ \exp\{\lambda X \} \vert \hist\right] \\
&\geq&\lambda \left( \E_{t}[R\left(Y_{t,a} \right)| A^*=a^*]- \E_{t}\left[ R\left(Y_{t,a} \right)\right] \right) - (\lambda^2 \sigma^2 /2).
\end{eqnarray*}
Maximizing over $\lambda$ yields the result. 
\end{proof}



\vskip 0.2in
\bibliography{references}

\end{document}